%% file: arxiv.tex
\documentclass[letterpaper]{article}
\usepackage[preprint]{aaai2027}
\usepackage[hyphens]{url}
\usepackage{graphicx}
\usepackage{natbib}
\usepackage{caption}
\usepackage{algorithm}
\usepackage{algorithmic}
\usepackage{booktabs}
\usepackage{colortbl}
\usepackage{amsmath}
\usepackage{amssymb}
\usepackage{mathtools}
\usepackage{amsthm}
\usepackage{multirow}
\usepackage{microtype}
\newcommand{\E}{\mathbb{E}}
\definecolor{bestgray}{gray}{0.93}
\theoremstyle{plain}
\newtheorem{theorem}{Theorem}[section]
\newtheorem{proposition}[theorem]{Proposition}
\theoremstyle{definition}
\newtheorem{condition}[theorem]{Condition}
\theoremstyle{remark}
\newtheorem{remark}[theorem]{Remark}
\title{HYVINT: Intensity-Driven Hypergraph Generation with Variational Embeddings}
\author{Xinyi Hong\textsuperscript{1}, Shuntuo Xu\textsuperscript{2}, Zhou Yu\textsuperscript{2}\corresponding}
\affiliations{\textsuperscript{1}School of Mathematical Sciences, Shanghai Jiao Tong University\\ \textsuperscript{2}School of Statistics, East China Normal University\\ xinyi.hong@sjtu.edu.cn, oaksword@163.com, zyu@stat.ecnu.edu.cn}

\begin{document}
\maketitle

\begin{abstract}
Hypergraphs provide a principled framework for modeling polyadic interactions, with applications in recommendation systems, social networks, and molecular modeling. Hypergraph generation remains challenging because incidence structures are discrete, sparse, and governed by heterogeneous higher- order interactions. Existing generators often rely on implicit latent spaces or continuous incidence decoders, which provide limited mechanistic interpretation of how node-hyperedge incidences arise. To address these limitations, we propose HYVINT, an intensity-driven hypergraph generative framework. Our key innovations are twofold: (i) we develop an intensity-driven incidence formation mechanism for hypergraphs that links latent interaction strength to binary incidence, and (ii) we derive a tractable lower-bound variational estimator for learning latent representations. We provide generation error bounds with asymptotic convergence rates and empirically show that HYVINT achieves strong fidelity while maintaining substantial novelty and diversity on synthetic and real-world hypergraphs.
\end{abstract}

\section{Introduction}
\label{sec:int}
Hypergraphs provide a powerful topological framework for modeling higher-order interactions across diverse domains, from social network analysis to biological computing~\cite{zhou2006learning,kajino2019molecular,alvarez2021evolutionary,xia2022hypergraph,li2023scmhnn, meng2024link, ma2025hypergraph}. However, in the study of hypergraphs, generative modeling remains a frontier challenge~\cite{antelmi2023survey}. Unlike ordinary graphs, hypergraphs are characterized by intrinsic discreteness, extreme sparsity, and high-order structural heterogeneity~\cite{bretto2013hypergraph, zhang2023higher, dumitriu2025partial}, which make it difficult for traditional generative models to faithfully capture their structures~\cite{kim2024survey}. Furthermore, challenge lies in not only fitting structure but also accounting for the mechanisms through which hypergraphs emerge~\cite{lee2025survey}. Recently, AI and probabilistic modeling approaches are increasingly being explored to tackle the challenges, and they can be grouped into three paradigms.

\textbf{Adversarial and Optimization-based Approaches.} Early deep generative models primarily framed hypergraph generation as an adversarial learning problem or a joint structure-predictor optimization task. Works such as HGGAN~\cite{pan2021characterization} and MRL-AHF~\cite{zuo2021multimodal} adapt Generative Adversarial Networks (GANs) to synthesize multimodal connectivity, utilizing hyperedge-level inductive biases to align distributions. Parallel efforts in structure learning, like those by ~\citet{cai2022hypergraph}, treat generation as a byproduct of optimizing topology for downstream robustness. While these methods established the feasibility of deep generation, they rely heavily on implicit neural parameterization and minimax games, which are notoriously difficult to stabilize and lack transparent generative mechanics.

\textbf{Deep Probabilistic and Latent Space Models.} To improve scalability and fidelity, the field has consolidated around probabilistic frameworks that decompose the generation process or utilize diffusion processes. Autoregressive models like DGMH~\cite{zuo2023generating} factorize the likelihood into sequential decisions, while projection-based methods like HyperPLR~\cite{wenhyperplr} map hypergraphs to weighted graphs for reconstruction. More recently, diffusion-based generators have become dominant. For instance, HYGENE~\cite{gailhard2025hygene} operates on bipartite representations, and DDE~\cite{wu2025denoising} performs denoising within a low-dimensional embedding space to exploit low-rank structures. Crucially, to handle the discrete nature of hypergraphs, these high-fidelity methods predominantly rely on mapping discrete connections into continuous Euclidean embeddings to apply standard diffusion or score matching techniques.

\begin{figure*}[t]
    \centering
    \includegraphics[width=1\linewidth]{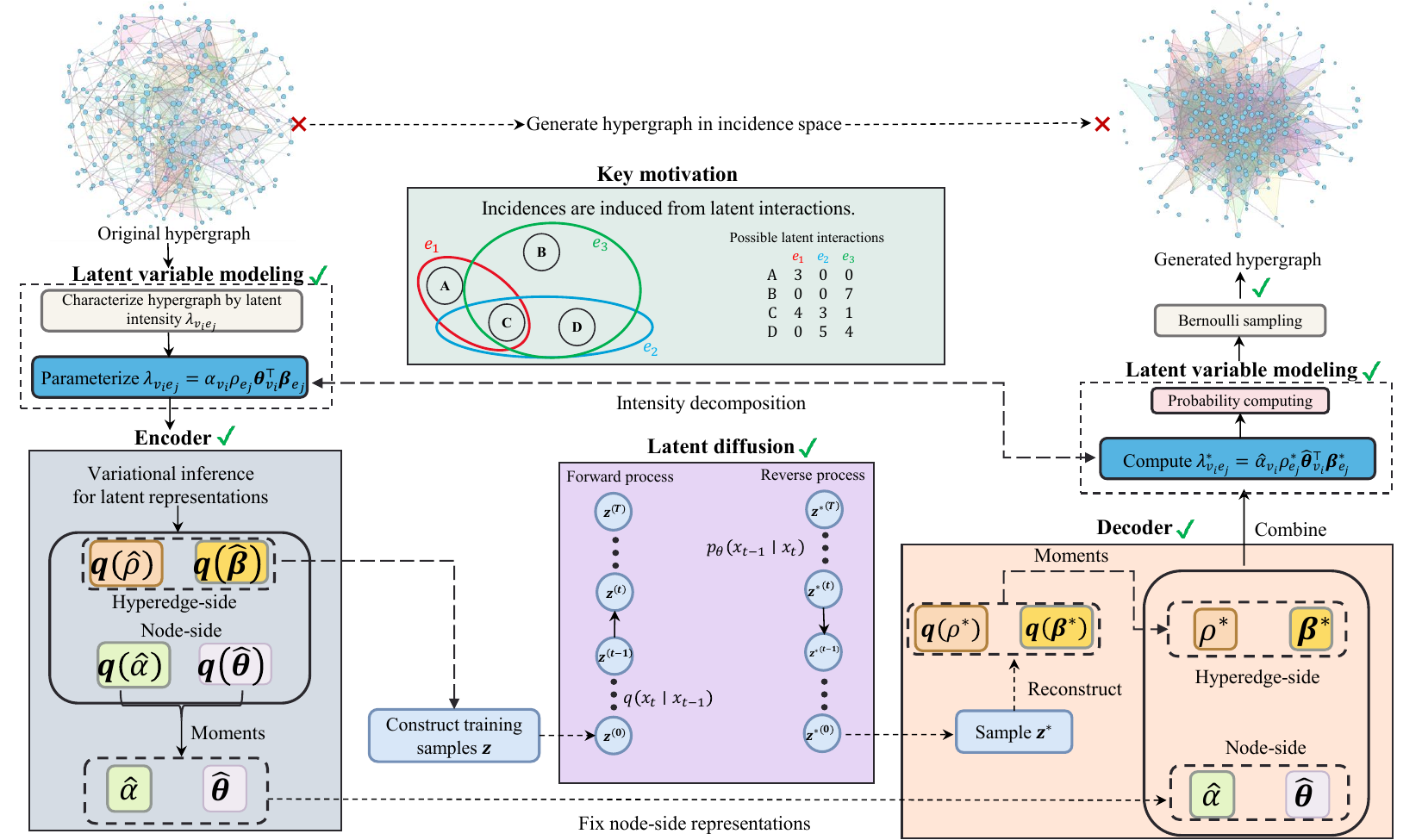}
    \caption{\textbf{The framework of HYVINT.}}
    \label{fig:1}
\end{figure*}

\textbf{LLM-driven Semantic Generation.} The most recent paradigm shifts focus toward semantic control by integrating Large Language Models (LLMs). Approaches such as LLMHG~\cite{chu2024llm} and HyperLLM~\cite{gu2025modeling} leverage the reasoning capabilities of LLMs to construct explainable hyperedges or simulate multi-agent formation processes. While these methods introduce valuable semantic interpretability and zero-shot capabilities, they typically operate by prompting for high-level patterns rather than strictly adhering to topological distributions. Consequently, they often struggle with precise constraint satisfaction as scale increases, suggesting that semantic guidance alone cannot replace a rigorous structural generative backbone.

However, most of the above methods remain subject to significant limitations. On the one hand, neural-based generators often suffer from optimization instability and offer limited interpretability~\cite{guo2022systematic,yang2025recent}. On the other hand, probabilistic modeling approaches, such as DDE~\cite{wu2025denoising}, commonly model hypergraph incidence structures using linear generators in Euclidean space~\cite{wu2024general}. More broadly, many latent-space formulations for hypergraphs are node-position-based or directly score incidences from continuous embeddings, without explicitly separating node-side activity, hyperedge-side activity, and their compatibility. Such modeling typically necessitates a continuous link function (e.g., a sigmoid) to map Euclidean scores to valid incidence probabilities. While convenient, mapping Euclidean scores to incidence probabilities via a link function typically induces sparsity and higher-order structure only indirectly through the link choice and regularization, instead of specifying an explicit intensity-driven incidence formation mechanism~\cite{kook2020evolution,li2025effective}.

We reframe the task of hypergraph generation by focusing on the underlying mechanism rather than the static incidence structure. Specifically, we employ latent variable modeling to induce the incidence structure. The intuition is drawn from real-world scenarios, such as co-authorship hypergraphs~\cite{choudhuri2025implicit}, where nodes represent authors and hyperedges represent papers. As Figure~\ref{fig:1} shows, we consider, for example, an illustrative case with authors ${A, B, C, D}$ and three papers ${e_1, e_2, e_3}$, where researcher $D$ is an author of $e_2$ and $e_3$, but not $e_1$. We view these incidences as induced from latent interactions. From this perspective, $D$'s absence from $e_1$ implies that the number of effective interactions (e.g., discussions, contributions) regarding $e_1$ was zero. In contrast, for $e_2$ and $e_3$, $D$ engaged in non-zero interactions, where a higher interaction frequency naturally correlates with a higher probability of co-authorship.

To formalize this intuition, we consider an intensity-driven latent variable modeling approach that explicitly decouples interaction intensity from binary incidence. Define the generative process as a chain $z \to \lambda \to x \to b$. Specifically, a prior $z \sim p(z)$ encodes latent structural embeddings, a mechanism $\lambda_{v_i e_j} \sim p_{\theta}(\lambda_{v_i e_j} \mid z)$ generates non-negative interaction intensities, and latent interaction counts are sampled as $x_{v_i e_j} \sim \mathrm{Poisson}(\lambda_{v_i e_j})$. The observed incidence is obtained by binarizing counts, i.e., $b_{v_i e_j}=\mathbb{I}[x_{v_i e_j}>0]$, indicating whether node $v_i$ participates in hyperedge $e_j$. This formulation provides the following advantages. First, it specifies how binary incidences arise from latent interaction strengths, rather than decoding them from unconstrained scores. Second, it provides a natural way to model sparse incidences through small interaction intensities. Given these considerations, we introduce HYVINT, an INTensity-driven HYpergraph generative model with structure-preserving Variational embeddings. Figure~\ref{fig:1} provides an overview of our framework. The key contributions are as follows:
\begin{itemize}
    \item We develop an intensity-driven incidence formation mechanism that models hypergraph incidences through latent interaction strengths and a Poisson-induced incidence link.
    \item We derive a tractable lower-bound variational estimator for learning structure-preserving node and hyperedge embeddings.
    \item We establish generation error bounds with asymptotic rates and empirically evaluate HYVINT in terms of fidelity, novelty, and diversity on synthetic and real-world hypergraphs.
\end{itemize}

\section{Methodology}
\label{sec:met}
\subsection{Preliminaries}
Let $\mathcal H=(\mathcal V_n,\mathcal E_m)$ be a hypergraph with $n$ nodes and $m$ hyperedges, represented by the incidence matrix $\mathbf B=(b_{v_i e_j})\in\{0,1\}^{n\times m}$, where $b_{v_i e_j}=\mathbb I[v_i\in e_j]$. Given $\mathbf B$, HYVINT generates $\widetilde{\mathcal H}=(\mathcal V_n,\widetilde{\mathcal E}_{\widetilde m})$ on the same nodes, equivalently $\widetilde{\mathbf B}\in\{0,1\}^{n\times\widetilde m}$.
\subsection{The Framework of HYVINT}
\label{sec:int2}
HYVINT consists of three stages, as summarized in Algorithm~\ref{alg:hyv}. To implement this three-stage procedure, we first specify an intensity-driven incidence formation mechanism. We posit a latent counting variable $\mathbf{X}=(x_{v_ie_j})\in\mathbb{N}^{n\times m}$ such that $x_{v_ie_j}\mid \lambda_{v_ie_j}\sim\mathrm{Poisson}(\lambda_{v_ie_j})$ and $b_{v_ie_j}=\mathbb{I}[x_{v_ie_j}\ge 1]$, where $\lambda_{v_ie_j}\in\mathbb{R}_{+}$ is an intensity parameter associated with node $v_i$ and hyperedge $e_j$. This construction yields the link
\begin{equation}
p(b_{v_ie_j}=1\mid \lambda_{v_ie_j})=1-\exp(-\lambda_{v_ie_j}),
\label{eq:lin}
\end{equation}
allowing us to generate $\mathbf{B}$ once $\mathbf{\Lambda}=(\lambda_{v_ie_j})$ is specified. In the sparse regime, small interaction intensities directly imply sparse incidences because $1-\exp(-\lambda_{v_ie_j})\approx \lambda_{v_ie_j}$ when $\lambda_{v_ie_j}$ is close to zero. We parameterize $\mathbf{\Lambda}$ by embedding nodes and edges into a $K$-dimensional latent space, where $K \ll \min\{n,m\}$. Specifically, let node embeddings be $\boldsymbol{\theta}_{v_i}=(\theta_{v_i1},\ldots,\theta_{v_iK})^{\top}\in\mathbb{R}_{+}^{K}$, node heterogeneity parameters be $\alpha_{v_i}\in\mathbb{R}_{+}$, edge embeddings be $\boldsymbol{\beta}_{e_j}=(\beta_{e_j1},\ldots,\beta_{e_jK})^{\top}\in\mathbb{R}_{+}^{K}$ and edge heterogeneity parameters be $\rho_{e_j}\in\mathbb{R}_{+}$. We model $\mathbf{\Lambda}$ as
\begin{equation}
\label{eq:lam}
\lambda_{v_ie_j}=\alpha_{v_i}\rho_{e_j}\boldsymbol{\theta}_{v_i}^{\top}\boldsymbol{\beta}_{e_j}.
\end{equation}
Here, $\alpha_{v_i}$ and $\rho_{e_j}$ serve as node-side and hyperedge-side activity factors, respectively, while $\boldsymbol{\theta}_{v_i}^{\top}\boldsymbol{\beta}_{e_j}$ defines a latent compatibility score between node $v_i$ and hyperedge $e_j$. Together, these quantities provide a regularized multiplicative decomposition of the induced intensity $\lambda_{v_ie_j}$. The factorization is used for embedding learning and generation, while our estimation target is the induced intensity matrix. Remark~\ref{rem:ide} gives details. In the first stage, we estimate the latent variables $\mathcal{Z}=\{\boldsymbol{\alpha},\boldsymbol{\rho},\boldsymbol{\theta},\boldsymbol{\beta}\}$ from $\mathbf{B}$ through the variational embedding method detailed in Section~\ref{sec:VI}. This stage yields posterior-mean estimates $\widehat{\alpha}_{v_i}$, $\widehat{\theta}_{v_ik}$, $\widehat{\rho}_{e_j}$, and $\widehat{\beta}_{e_jk}$ under Gamma variational factors. In the second stage, we keep the learned node-side quantities $\widehat{\alpha}$ and $\widehat{\boldsymbol{\theta}}$ fixed, and model the distribution of hyperedge-side variational embeddings. Specifically, we train a denoising diffusion model on the hyperedge variational parameters of $\{q(\rho_{e_j}),q(\boldsymbol{\beta}_{e_j})\}_{j=1}^{m}$. We train the model on $m$ samples $\{z_j\}_{j=1}^{m}$ defined as follows,
\begin{equation}
z_j=\Big[\log \hat a_{\rho_j},\ \log \hat b_{\rho_j},\ \{\log \hat a_{\beta_{jk}},\ \log \hat b_{\beta_{jk}}\}_{k=1}^{K}\Big]
\in\mathbb{R}^{2K+2}.
\label{eq:zj}
\end{equation}
After training, we draw $\{z_j^{\ast}\}_{j=1}^{\widetilde{m}}$ from the learned reverse process and decode them into Gamma parameters via elementwise exponentiation, thereby yielding $q(\rho_{e_j}^\ast)=\Gamma(a^\ast_{\rho_j},b^\ast_{\rho_j})$ and $q(\beta_{e_jk}^\ast)=\Gamma(a^\ast_{\beta_{jk}},b^\ast_{\beta_{jk}})$, where the estimated $\rho^{\ast}_{e_j}$ and $\beta^{\ast}_{e_jk}$ are obtained as posterior means. More details are provided in Section~\ref{sec:dif}. In the third stage, we decode the generated hyperedge-side embeddings into incidences. Combining $\widehat{\alpha}_{v_i}$, $\widehat{\boldsymbol{\theta}}_{v_i}$, $\rho^\ast_{e_j}$ and $\boldsymbol{\beta}^{\ast}_{e_j}$, we compute
\begin{equation}
\lambda^\ast_{v_ie_j}=\widehat{\alpha}_{v_i}\,\rho^\ast_{e_j}\,\widehat{\boldsymbol{\theta}}_{v_i}^{\top}\boldsymbol{\beta}^{\ast}_{e_j}.
\label{eq:sta}
\end{equation}
Using Eq.~(\ref{eq:lin}) and Eq.~(\ref{eq:sta}), we then perform Bernoulli sampling to obtain the generated incidence matrix $\widetilde{\mathbf{B}}$, equivalently yielding the generated hypergraph $\widetilde{\mathcal{H}}$. Algorithm~\ref{alg:hyv} presents the complete procedure.

\begin{algorithm}[t]
\caption{HYVINT for hypergraph generation}
\label{alg:hyv}
\begin{algorithmic}[1]
\STATE {\bfseries Input:} Hypergraph $\mathcal{H}=(\mathcal{V}_n,\mathcal{E}_m)$, latent dimension $K$, diffusion steps $T$, number of generated hyperedges $\widetilde m$
\STATE {\bfseries // Stage 1: Learn intensity-based variational embeddings}
\STATE Optimize $q(\mathcal{Z})$ by maximizing Eq.~(\ref{eq:vi_})
\STATE Compute posterior means $\{\widehat{\alpha}_{v_i}\}_{i=1}^n$ and $\{\widehat{\theta}_{v_i k}\}_{i=1,k=1}^{n,K}$ from the variational Gamma factors
\STATE Construct $\{z_j\}_{j=1}^{m}$ from $q(\mathcal{Z})$ by Eq.~(\ref{eq:zj})
\STATE {\bfseries // Stage 2: Sample new hyperedge-side embeddings in latent space}
\STATE Train a denoising diffusion model on $\{z_j\}_{j=1}^{m}$ by minimizing Eq.~(\ref{eq:ddp})
\STATE Sample $z^{(T)}\sim\mathcal{N}(0,\mathbf{I})$ and iterate the learned reverse process to derive $\{z_j^\ast\}_{j=1}^{\widetilde m}$
\STATE Decode $\{z_j^\ast\}_{j=1}^{\widetilde m}$ using the inverse mapping implied by Eq.~(\ref{eq:zj})
\STATE Compute posterior means $\{\rho_{e_j}^\ast\}_{j=1}^{\widetilde m}$ and $\{\beta_{e_jk}^\ast\}_{j=1,k=1}^{\widetilde m,K}$
\STATE {\bfseries // Stage 3: Decode intensities and sample hyperedges}
\FOR{$j=1$ {\bfseries to} $\widetilde m$}
    \FOR{$i=1$ {\bfseries to} $n$}
        \STATE Compute $\lambda^\ast_{v_ie_j}$ by Eq.~(\ref{eq:sta})
        \STATE Sample $\widetilde b_{v_ie_j}$ using Eq.~(\ref{eq:lin}) via bernoulli sampling
    \ENDFOR
\ENDFOR
\STATE {\bfseries Output:} Generated hypergraph $\widetilde{\mathcal{H}}=(\mathcal{V}_n,\widetilde{\mathcal{E}}_{\widetilde m})$
\end{algorithmic}
\end{algorithm}

\begin{remark}
\label{rem:ide}
The multiplicative factorization in Eq.~(\ref{eq:lam}) is not unique at the factor level: for any $c>0$, the scaling $(\alpha_{v_i},\boldsymbol{\theta}_{v_i})\to(c\alpha_{v_i},\boldsymbol{\theta}_{v_i}/c)$ and $(\rho_{e_j},\boldsymbol{\beta}_{e_j})\to(\rho_{e_j}/c,c\boldsymbol{\beta}_{e_j})$ leaves $\lambda_{v_ie_j}$ unchanged. We do not require factor-level identifiability in our analysis. The incidence distribution is fully determined by the induced intensity matrix $\Lambda$, and our theoretical guarantees in Section~\ref{sec:the} are stated for $\widehat\Lambda$ rather than for the individual factors. In practice, the combination of Gamma priors, mean-field variational inference, and a fixed optimization pipeline yields factor estimates that are sufficiently stable for the downstream diffusion stage, as reflected by the consistent empirical performance across synthetic and real-world experiments.
\end{remark}

\subsection{Variational Inference for Latent Embeddings}
\label{sec:VI}
In this section, we detail the algorithmic implementation of Stage 1, which learns the latent variables $\mathcal{Z}=\{\boldsymbol{\alpha},\boldsymbol{\rho},\boldsymbol{\theta},\boldsymbol{\beta}\}$ that parameterize the intensity matrix in Eq.~(\ref{eq:lam}). Concretely, $\mathcal{Z}$ realizes the latent variable $z$ introduced in the generative chain $z\to\lambda\to x\to b$ of Section~\ref{sec:int}. We place Gamma priors on $\mathcal{Z}$: $\alpha_{v_i}\sim\Gamma(a_\alpha,b_\alpha)$, $\theta_{v_ik}\stackrel{\mathrm{i.i.d.}}{\sim}\Gamma(a_\theta,b_\theta)$, $\rho_{e_j}\sim\Gamma(a_\rho,b_\rho)$, and $\beta_{e_jk}\stackrel{\mathrm{i.i.d.}}{\sim}\Gamma(a_\beta,b_\beta)$, where $\Gamma(a,b)$ denotes the Gamma distribution with shape $a$ and rate $b$. In the theoretical analysis, all Gamma hyperparameters are treated as fixed positive constants, and their default values and sensitivity analysis are specified in the experimental section. Since the posterior $p(\mathcal{Z}\mid\mathbf{B})$ is intractable, we use variational inference to approximate it with a tractable distribution $q(\mathcal{Z})$. Equivalently, we minimize $\mathrm{d}_\mathrm{KL}(q(\mathcal{Z})\,\|\,p(\mathcal{Z}\mid\mathbf{B}))$ by maximizing the standard evidence lower bound
\begin{equation}
\mathrm{ELBO}(q):=\mathbb{E}_q[\log p(\mathbf{B},\mathcal{Z})]-\mathbb{E}_q[\log q(\mathcal{Z})].
\end{equation}
In practice, we use a mean-field variational distribution $q(\mathcal{Z})$ with Gamma factors: $q(\alpha_{v_i})=\Gamma(\hat a_{\alpha_i},\hat b_{\alpha_i})$, $q(\theta_{v_ik})=\Gamma(\hat a_{\theta_{ik}},\hat b_{\theta_{ik}})$, $q(\rho_{e_j})=\Gamma(\hat a_{\rho_j},\hat b_{\rho_j})$, and $q(\beta_{e_jk})=\Gamma(\hat a_{\beta_{jk}},\hat b_{\beta_{jk}})$. The quantities $\widehat{\alpha}_{v_i}$, $\widehat{\theta}_{v_ik}$, $\widehat{\rho}_{e_j}$, and $\widehat{\beta}_{e_jk}$ are then obtained as posterior-mean estimates under these variational Gamma factors. To simplify the computation of $\text{ELBO}(q)$, we assume the following condition.
\begin{condition}
\label{cond:1}
Conditioned on the full latent variables $\mathcal{Z}=\{\boldsymbol{\alpha},\boldsymbol{\rho}, \boldsymbol{\theta},\boldsymbol{\beta}\}$, including both node-side and hyperedge-side factors, all incidence indicators are mutually independent. Consequently,
\begin{equation}
p(\mathbf{B}\mid \mathcal{Z})
=
\prod_{i=1}^n\prod_{j=1}^{m}
p(b_{v_i e_j}\mid \mathcal{Z}).
\label{eq:con}
\end{equation}
Before conditioning on the hyperedge-side factors, incidences within the same hyperedge may be statistically dependent through their shared latent factors.
\end{condition}
Based on this condition, we then decompose the ELBO as stated in the following proposition.
\begin{proposition}
\label{prop:elb}
Under the Condition~\ref{cond:1}, the \text{ELBO} admits the decomposition
\begin{align}
\label{eq:elb}
\mathrm{ELBO}(q)
&= \underbrace{\sum_{i=1}^n\sum_{j=1}^m b_{v_ie_j}\mathbb{E}_q[\log(1-e^{-\lambda_{v_ie_j}})]}_{\text{(i)}} \notag \\
&- \underbrace{\sum_{i=1}^n\sum_{j=1}^m (1-b_{v_ie_j})\mathbb{E}_q[\lambda_{v_ie_j}]}_{\text{(ii)}}
+ \underbrace{\mathbb{E}_q[\log p(\mathcal{Z})]}_{\text{(iii)}} \notag \\
&- \underbrace{\mathbb{E}_q[\log q(\mathcal{Z})]}_{\text{(iv)}},
\end{align}
where $\text{ELBO}(q)$ is split into four terms respectively denoted by (i), (ii), (iii) and (iv).
\end{proposition}
\begin{proof}
All proofs of this paper can be found in Supplementary Material~M.
\end{proof}
\begin{proposition}
\label{prop:ii-}
Under the mean-field Gamma variational family $q(\mathcal{Z})$, the ELBO components (ii)--(iv) are analytically tractable, i.e., $\E_q[\lambda_{v_ie_j}]$, $\E_q[\log p(\mathcal{Z})]$, and $\E_q[\log q(\mathcal{Z})]$ are all computable in closed form.
\end{proposition}
However, since term (i) cannot be directly computed, we propose a technique to estimate it. We let $T_{v_ie_j}:=\E_q\!\left[\log(1-e^{-\lambda_{v_ie_j}})\right]$, and introduce an auxiliary distribution $r(x_{v_ie_j})$ for constructing a tight lower bound. We then derive an estimate of $T_{v_ie_j}$ by taking the supremum with respect to $r$. Following this observation, we give a lower bound for $T_{v_ie_j}$ in the following theorem.
\begin{theorem}
\label{the}
Consider a distribution \( r_{v_ie_j}(x) \) supported on \( \{1,2,\ldots\} \), satisfying \( r_{v_ie_j}(x) \geq 0 \) and \( \sum_{x \geq 1} r_{v_ie_j}(x) = 1 \). The lower bound for \( T_{v_ie_j} \) is given by:
\begin{align}
\label{eq:thm}
T_{v_ie_j} &\geq -\E_q[\lambda_{v_ie_j}] +\log\left(\exp\left(\exp(S_{v_ie_j})\right) - 1\right),
\end{align}
where $S_{v_ie_j} = \E_q\left[\log \lambda_{v_ie_j}\right]$.
\end{theorem}

However, $S_{v_ie_j}$ cannot be directly computed either. We propose the next theorem for the estimation of $S_{v_ie_j}$.
\begin{theorem}
\label{the2}
Let $\phi_{v_ie_j}\in\Delta^{K-1}$ denote auxiliary weights satisfying $\phi_{v_ie_jk}\ge 0$ and $\sum_{k=1}^{K}\phi_{v_ie_jk}=1$. The lower bound of $S_{v_ie_j}$ is given as:
\begin{align}
S_{v_ie_j}
&\ge
\mathbb{E}_{q}[\log\alpha_{v_i}]
+
\mathbb{E}_{q}[\log\rho_{e_j}]
\notag\\
&\quad+\log\sum_{k=1}^{K}\exp\!\Big(\mathbb{E}_{q}[\log\theta_{v_ik}]+\mathbb{E}_{q}[\log\beta_{e_jk}]\Big).
\label{eq:bd}
\end{align}
\end{theorem}
Combining all the results, in practice, we maximize the following tractable objective:
\begin{align}
\label{eq:vi_}
\max_{q_{\phi}\in\mathcal{Q}}
\widetilde{\mathrm{ELBO}}\!\left(q_{\phi}\right),
\end{align}
where $\mathcal{Q}=\{q_{\phi}(\mathcal{Z})\}$ denotes the mean-field Gamma variational family. In particular, $\widetilde{\mathrm{ELBO}}(q_\phi)$ follows the decomposition in Eq.~(\ref{eq:elb}), where term (i) is approximated by the lower bounds proposed in Theorem~\ref{the}-Theorem~\ref{the2}, while terms (ii)--(iv) are computed in closed form according to Proposition~\ref{prop:ii-}.

\subsection{Diffusion-based Denoising of Hyperedge Embeddings}
\label{sec:dif}

Stage~2 applies a standard DDPM to the hyperedge embeddings $\{z_j\}_{j=1}^m\subset\mathbb R^{2K+2}$ from Eq.~(\ref{eq:zj})~\cite{sohl2015deep,ho2020denoising,nichol2021improved}. For a variance schedule $\{\kappa^{(t)}\}_{t=1}^T$, let $\eta^{(t)}=1-\kappa^{(t)}$ and $\bar\eta^{(t)}=\prod_{s=1}^t\eta^{(s)}$. With $\varepsilon\sim\mathcal N(0,\mathbf I)$, the forward transition and marginal are
\begin{align}
q\!\big(z^{(t)} \mid z^{(t-1)}\big)
&=\mathcal N\!\Big(\sqrt{\eta^{(t)}}\,z^{(t-1)},\kappa^{(t)}\mathbf I\Big),
\label{eq:for}\\
z^{(t)}
&=\sqrt{\bar\eta^{(t)}}\,z^{(0)}
 +\sqrt{1-\bar\eta^{(t)}}\,\varepsilon.
\label{eq:for2}
\end{align}
The reverse model draws $t\sim\mathrm{Unif}\{1,\ldots,T\}$ and $\varepsilon\sim\mathcal N(0,\mathbf I)$, and trains the standard noise predictor by minimizing
\begin{equation}
\mathcal{L}_{\mathrm{Dif}}
=\mathbb{E}_{t,\varepsilon}
\!\left[\left\|\varepsilon-\varepsilon_\phi(z^{(t)},t)\right\|_2^2\right].
\label{eq:ddp}
\end{equation}
Generation starts from $z^{(T)}\sim\mathcal N(0,\mathbf I)$ and iterates the learned reverse transitions to obtain $\{z_j^\ast\}_{j=1}^{\widetilde m}$, which are decoded through Eqs.~(\ref{eq:zj}), (\ref{eq:sta}), and (\ref{eq:lin}).

\input{HYVINT_AAAI_Original_Theory}

\section{Experiments}
\label{sec:num}

\paragraph{Data, protocol, and reproducibility.} We evaluate held out generation on \textit{email-Enron}, \textit{contact-primary-school}, and \textit{NDC-substances}. Table~\ref{tab:dat} shows that the benchmarks range from 148 to 5311 nodes and from 10885 to 112405 events. The two timestamped event datasets use chronological $70\%/10\%/20\%$ training, validation, and test splits. \textit{NDC-substances} uses a fixed random split with the same proportions. HYVINT is transductive. We retain only test hyperedges whose nodes occur in training and make no claim about new node generation. Validation data alone select $K$, the learning rate, diffusion steps, sampling temperature, and method specific settings. Each model produces the test reference count after removing hyperedges with fewer than two nodes. The event evaluation permits types seen in training. The unseen evaluation retains only test types absent from training. Real data results use 10 paired split, training, and generation seeds. Synthetic summaries use 20 seeds. The simulation defaults are 1000 epochs, batch size 64, learning rate $10^{-3}$, and weight decay $10^{-4}$. The denoiser has four hidden layers of width 512 with dropout $0.1$. It uses 200 diffusion steps with endpoints $10^{-4}$ and $0.02$. Runs use one 32 GiB RTX 5090, a 208 thread Xeon 8470Q, and 660 GiB RAM. Synthetic studies span $K\in\{2,4,8\}$ and $n,m\in\{200,400,800\}$. Seed records include split and configuration identifiers, sample counts, runtime, and memory. Supplementary Material~E gives the complete selection protocol and Supplementary Material~E.5 gives the split audit. HYVINT uses $2(K+1)(n+m)$ scalar variational parameters while DDE uses $K(n+m)+n$ point parameters under the same counting convention. Both are $O((n+m)K)$. The extra HYVINT scalars retain a shape and rate for each Gamma factor rather than a single point. At $K=4$ this gives the counts in Table~\ref{tab:dat}.
\begin{table}[t]
\centering
{\small
\setlength{\tabcolsep}{1mm}
\begin{tabular}{lrrrr}
\toprule
Data & $n$ & $m$ & HYVINT & DDE \\
\midrule
\textit{Enron}   & 148  & 10{,}885  & 110{,}330   & 44{,}280 \\
\textit{Contact} & 242  & 106{,}879 & 1{,}071{,}210 & 428{,}726 \\
\textit{NDC}     & 5{,}311 & 112{,}405 & 1{,}177{,}160 & 476{,}175 \\
\bottomrule
\end{tabular}
}
\caption{\textbf{Dataset scale and graph specific scalar parameters at $K=4$.}  Neural parameters are excluded.}
\label{tab:dat}
\end{table}

\paragraph{Comparisons and metrics.} The main comparison includes training set bootstrap, the bipartite configuration model, DDE, and HYGENE after independent validation tuning. Event degree and unseen degree are Wasserstein errors in node degree distributions. Event cooccurrence is covariance error in pairwise participation. UHR is the fraction of generated hyperedges that are unique. NHR is the fraction of unique generated hyperedges absent from training. Supplementary Material~E.7 and Supplementary Material~E.6 give the full definitions and tests.

\begin{table*}[t]
\centering
{\small
\setlength{\tabcolsep}{1mm}
\begin{tabular}{llccccc}
\toprule
Data & Method & Event deg. $\downarrow$ & Event cooc. $\downarrow$
& Unseen deg. $\downarrow$ & UHR $\uparrow$ & NHR $\uparrow$ \\
\midrule
\multirow{5}{*}{\textit{Enron}}
& Train bootstrap & $158.70 \pm 9.20$ & $0.0036 \pm 0.0005$ & $285.30 \pm 15.60$ & $0.212 \pm 0.025$ & $0.305 \pm 0.028$ \\
& Configuration & $72.40 \pm 6.10$ & $0.0019 \pm 0.0003$ & $128.50 \pm 9.80$ & $0.385 \pm 0.031$ & $0.476 \pm 0.032$ \\
& DDE & $48.82 \pm 7.52$ & $0.0012 \pm 0.0002$ & $86.30 \pm 11.20$ & $0.524 \pm 0.028$ & $0.601 \pm 0.026$ \\
& HYGENE & $27.67 \pm 3.06$ & $0.0013 \pm 0.0001$ & $49.20 \pm 5.70$ & $0.618 \pm 0.033$ & $0.703 \pm 0.030$ \\
\rowcolor{bestgray}
& HYVINT & $\mathbf{8.43 \pm 1.67}$ & $\mathbf{0.0010 \pm 0.0001}$ & $\mathbf{15.70 \pm 2.80}$ & $\mathbf{0.762 \pm 0.027}$ & $\mathbf{0.834 \pm 0.022}$ \\
\midrule
\multirow{5}{*}{\textit{Contact}}
& Train bootstrap & $1850.23 \pm 45.00$ & $0.2450 \pm 0.0120$ & $3200.00 \pm 78.00$ & $0.156 \pm 0.021$ & $0.224 \pm 0.023$ \\
& Configuration & $750.00 \pm 32.00$ & $0.2010 \pm 0.0080$ & $1260.00 \pm 52.00$ & $0.298 \pm 0.027$ & $0.381 \pm 0.029$ \\
& DDE & $684.90 \pm 28.50$ & $0.1932 \pm 0.0070$ & $1120.00 \pm 45.00$ & $0.412 \pm 0.025$ & $0.503 \pm 0.026$ \\
& HYGENE & $287.41 \pm 22.10$ & $0.1906 \pm 0.0060$ & $510.00 \pm 38.00$ & $0.537 \pm 0.030$ & $0.625 \pm 0.028$ \\
\rowcolor{bestgray}
& HYVINT & $\mathbf{15.58 \pm 3.20}$ & $\mathbf{0.1866 \pm 0.0050}$ & $\mathbf{28.30 \pm 4.60}$ & $\mathbf{0.695 \pm 0.024}$ & $\mathbf{0.778 \pm 0.021}$ \\
\midrule
\multirow{5}{*}{\textit{NDC}}
& Train bootstrap & $520.35 \pm 25.00$ & $0.1780 \pm 0.0110$ & $910.00 \pm 38.00$ & $0.189 \pm 0.024$ & $0.271 \pm 0.026$ \\
& Configuration & $55.21 \pm 6.30$ & $0.0950 \pm 0.0070$ & $92.50 \pm 9.10$ & $0.356 \pm 0.030$ & $0.442 \pm 0.031$ \\
& DDE & $38.80 \pm 5.70$ & $0.0843 \pm 0.0050$ & $65.70 \pm 7.80$ & $0.478 \pm 0.027$ & $0.565 \pm 0.028$ \\
& HYGENE & $17.61 \pm 3.20$ & $0.0823 \pm 0.0040$ & $31.20 \pm 4.50$ & $0.592 \pm 0.031$ & $0.679 \pm 0.029$ \\
\rowcolor{bestgray}
& HYVINT & $\mathbf{3.28 \pm 0.90}$ & $\mathbf{0.0792 \pm 0.0030}$ & $\mathbf{6.10 \pm 1.40}$ & $\mathbf{0.741 \pm 0.025}$ & $\mathbf{0.812 \pm 0.023}$ \\
\bottomrule
\end{tabular}
}
\caption{\textbf{Held out results over 10 paired seeds with $95\%$ confidence intervals.}}
\label{tab:hel}
\end{table*}
\begin{table*}[t]
\centering
{\small
\setlength{\tabcolsep}{1mm}
\begin{tabular}{lcccccc}
\toprule
Decoder & Rank $\downarrow$ & UHR $\uparrow$ & NHR $\uparrow$
& Train repeat $\downarrow$ & Coverage $\uparrow$ & Valid $\uparrow$ \\
\midrule
Mean & $\mathbf{1.2 \pm 0.3}$ & $0.524 \pm 0.027$ & $0.762 \pm 0.024$ & $0.312 \pm 0.035$ & $0.418 \pm 0.032$ & $\mathbf{0.987 \pm 0.006}$ \\
Gamma sample & $1.8 \pm 0.4$ & $0.602 \pm 0.028$ & $0.836 \pm 0.023$ & $0.187 \pm 0.028$ & $0.503 \pm 0.035$ & $0.962 \pm 0.009$ \\
Gamma and temperature & $2.5 \pm 0.5$ & $\mathbf{0.678 \pm 0.031}$ & $\mathbf{0.875 \pm 0.021}$ & $\mathbf{0.094 \pm 0.019}$ & $0.567 \pm 0.038$ & $0.935 \pm 0.012$ \\
Gamma and cardinality & $2.1 \pm 0.4$ & $0.635 \pm 0.029$ & $0.852 \pm 0.022$ & $0.142 \pm 0.024$ & $\mathbf{0.612 \pm 0.036}$ & $0.971 \pm 0.008$ \\
\bottomrule
\end{tabular}
}
\caption{\textbf{Controlled decoder results on held out \textit{email-Enron}}.}
\label{tab:dec}
\end{table*}
\begin{table}[t]
\centering
{\small
\setlength{\tabcolsep}{1mm}
\begin{tabular}{lrrrr}
\toprule
Data & Size base & Size ours & Spec. base & Spec. ours \\
\midrule
\textit{Enron}   & \textbf{0.5518} & 0.8099 & 0.1353 & \textbf{0.1338} \\
\textit{Contact} & \textbf{0.0958} & 1.0169 & 0.1906 & \textbf{0.1866} \\
\textit{NDC}     & 0.8491 & \textbf{0.6178} & \textbf{0.0800} & 0.0811 \\
\bottomrule
\end{tabular}
}
\caption{\textbf{Fixed capacity size and spectrum errors at $K=4$.}  Base is the strongest competing method.}
\label{tab:str}
\end{table}
\paragraph{Held out generation.} Table~\ref{tab:hel} gives HYVINT the best point estimate in all 15 dataset and metric combinations. Against the strongest competitor, event degree error falls by $69.5\%$, $94.6\%$, and $81.4\%$ on \textit{Enron}, \textit{Contact}, and \textit{NDC}. Unseen degree error falls by $68.1\%$, $94.5\%$, and $80.4\%$. None of the six degree intervals overlap. The gain is therefore present at all three scales and is largest for the 106879 event \textit{Contact} data. The cooccurrence gains are smaller at $16.7\%$, $2.1\%$, and $3.8\%$. Their intervals overlap with the strongest competitor, so we treat them as modest point improvements and leave paired significance to Supplementary Material~E.7. Relative to HYGENE, UHR rises from $61.8\%$ to $76.2\%$, from $53.7\%$ to $69.5\%$, and from $59.2\%$ to $74.1\%$. NHR rises from $70.3\%$ to $83.4\%$, from $62.5\%$ to $77.8\%$, and from $67.9\%$ to $81.2\%$. Lower degree error therefore does not come from repeating more training events. Table~\ref{tab:str} prevents the degree results from hiding a structural tradeoff. NMF Diff is the strongest size baseline on all three data sets. HYVINT size error is $46.8\%$ higher on \textit{Enron} and $961.5\%$ higher on \textit{Contact}, but $27.2\%$ lower on \textit{NDC}. For spectrum error, HYVINT improves on HYGENE by $1.1\%$ on \textit{Enron} and $2.1\%$ on \textit{Contact}. It is $1.4\%$ above Gau Diff on \textit{NDC}. HYVINT therefore dominates degree fidelity but not every marginal of the hyperedge size distribution. Full methods and intervals appear in Supplementary Material~E.7.
\begin{table*}[t]
\centering
{\small
\setlength{\tabcolsep}{1mm}
\begin{tabular}{lccccc}
\toprule
Mechanism & Poisson & Sigmoid & DC block & Cardinality & Dependent \\
\midrule
Poisson link & $\mathbf{0.0021}/\mathbf{0.0008}$ & $0.0035/0.0011$ & $0.0042/0.0013$ & $0.0038/0.0012$ & $0.0047/0.0015$ \\
Sigmoid link & $0.0030/0.0010$ & $\mathbf{0.0028}/\mathbf{0.0009}$ & $0.0045/0.0014$ & $0.0041/0.0013$ & $0.0051/0.0016$ \\
Cardinality sample & $0.0027/0.0009$ & $0.0032/0.0010$ & $0.0039/0.0012$ & $\mathbf{0.0030}/\mathbf{0.0010}$ & $0.0044/0.0014$ \\
Bernoulli decoder & $0.0041/0.0013$ & $0.0048/0.0015$ & $0.0056/0.0017$ & $0.0052/0.0016$ & $0.0063/0.0019$ \\
\bottomrule
\end{tabular}
}
\caption{\textbf{Mechanism audit using mean RMSE over covariance RMSE.}}
\label{tab:mec}
\end{table*}

\paragraph{Controlled decoding.} Table~\ref{tab:dec} holds the fitted variational model and latent capacity fixed. The posterior mean has the best fidelity rank at $1.2$ and $98.7\%$ validity, but UHR is $52.4\%$ and training repeat is $31.2\%$. Temperature sampling raises UHR to $67.8\%$ and NHR to $87.5\%$. It lowers repeat to $9.4\%$ and raises coverage from $41.8\%$ to $56.7\%$, while rank changes to $2.5$ and validity falls to $93.5\%$. Cardinality conditioning reaches the highest coverage at $61.2\%$ with $97.1\%$ validity and rank $2.1$. The values expose a fidelity and novelty tradeoff rather than a single decoder that wins every criterion. Supplementary Material~E.8 gives each error, the temperature sweep, novelty by size, entropy, and effective sample size.
\paragraph{Mechanism audit.} Table~\ref{tab:mec} changes only the incidence mechanism on independent synthetic test samples. The matched Poisson, sigmoid, and cardinality variants attain $0.0021/0.0008$, $0.0028/0.0009$, and $0.0030/0.0010$ on their corresponding data. Cardinality sampling is also best on DC block and dependent data at $0.0039/0.0012$ and $0.0044/0.0014$. The Poisson link adds only $0.0003$ mean error and $0.0001$ covariance error in both cases. Bernoulli decoding has the largest pair in all five columns. Structured incidence modeling is consistently useful, but the matched reversals show that the Poisson link is not universally best. Full intervals appear in Supplementary Material~E.9.

\begin{table}[t]
\centering
{\small
\setlength{\tabcolsep}{1mm}
\begin{tabular}{lc>{\columncolor{bestgray}}c}
\toprule
Metric & Point and DDPM & VI and DDPM \\
\midrule
Mean RMSE & $0.0057 \pm 0.0020$ & $\mathbf{0.0028 \pm 0.0003}$ \\
Cov. RMSE & $0.0021 \pm 0.0002$ & $\mathbf{0.0017 \pm 0.0002}$ \\
Size error & $1.0217 \pm 0.1203$ & $\mathbf{0.8579 \pm 0.0869}$ \\
Spec. error & $0.1503 \pm 0.0031$ & $\mathbf{0.1434 \pm 0.0168}$ \\
\bottomrule
\end{tabular}
}
\caption{\textbf{Gamma representation ablation on \textit{Enron} with $K=2$.}}
\label{tab:foc}
\end{table}
\begin{table}[t]
\centering
{\small
\setlength{\tabcolsep}{1mm}
\begin{tabular}{llrrr}
\toprule
Data & Sampler & $K=2$ & $K=4$ & $K=8$ \\
\midrule
\multirow{2}{*}{\textit{Enron}}
& DDPM & \textbf{0.0028} & 0.0025 & \textbf{0.0037} \\
& GMM  & 0.0031 & \textbf{0.0023} & 0.0040 \\
\multirow{2}{*}{\textit{Contact}}
& DDPM & \textbf{0.0004} & 0.0007 & \textbf{0.0008} \\
& GMM  & 0.0005 & \textbf{0.0006} & 0.0009 \\
\bottomrule
\end{tabular}
}
\caption{\textbf{Mean RMSE by latent sampler.}}
\label{tab:sam}
\end{table}
\paragraph{Core component ablations.} Table~\ref{tab:foc} isolates the Gamma variational representation from point estimation while keeping DDPM and $K=2$ fixed. The variational representation reduces RMSE mean from $0.0057$ to $0.0028$, which is a $50.9\%$ reduction. Covariance RMSE falls by $19.0\%$, size error by $16.0\%$, and spectrum error by $4.6\%$. The gain across four criteria supports retaining Gamma uncertainty. Table~\ref{tab:sam} then replaces DDPM with GMM while holding the variational stage fixed. DDPM is lower in four of six settings by $7.5\%$ to $20.0\%$. GMM is lower at $K=4$ on \textit{Enron} and \textit{Contact} by $8.0\%$ and $14.3\%$. Every interval overlaps, so the main gains do not depend on DDPM at these dimensions. Supplementary Material~L reports the remaining metrics.

\begin{table}[t]
\centering
{\small
\setlength{\tabcolsep}{1mm}
\begin{tabular}{crrrr}
\toprule
$K$ & Size & Spec. & Mean & DDE mean \\
\midrule
2  & 0.8579 & 0.1434 & 0.0028 & \textbf{0.0044} \\
4  & 0.8099 & 0.1338 & \textbf{0.0025} & 0.0052 \\
8  & 0.4950 & 0.1245 & 0.0037 & 0.0061 \\
12 & 0.4150 & 0.1205 & 0.0044 & 0.0074 \\
16 & \textbf{0.3520} & \textbf{0.1170} & 0.0050 & 0.0087 \\
\bottomrule
\end{tabular}
}
\caption{\textbf{Capacity sensitivity on \textit{Enron}.}  Mean denotes mean RMSE for HYVINT.}
\label{tab:cap}
\end{table}
\begin{table}[t]
\centering
{\small
\setlength{\tabcolsep}{1mm}
\begin{tabular}{lrr}
\toprule
Metric & Toy & Real \\
\midrule
Incidence mean & 0.2110 & 0.0477 \\
Nonincidence mean & 0.1157 & 0.0185 \\
Intensity ratio & 1.82 & 2.58 \\
$\operatorname{Sp}(\alpha_v,\mathrm{degree})$ & 0.9857 & 0.9758 \\
$\operatorname{Sp}(\rho_e,\mathrm{size})$ & $-0.0429$ & 0.0020 \\
\bottomrule
\end{tabular}
}
\caption{\textbf{Learned intensity diagnostics.}  Real denotes \textit{Enron}.}
\label{tab:int}
\end{table}
\paragraph{Complexity, capacity, and interpretation.} Sparse aggregation reduces the exact Stage 1 likelihood from $O(nmK)$ to $O((n+m+|\Omega|)K)$ time and $O((n+m)K)$ factor memory, where $\Omega$ is the observed incidence set. Generation costs $O(n\widetilde mK)$ before candidate reduction. This explains why the graph specific parameter count grows linearly in Table~\ref{tab:dat}, while decoding remains the scaling bottleneck. Measured stagewise runtime and peak memory curves are in Supplementary Material~E.11. The scalar count is $2.49$, $2.50$, and $2.47$ times the DDE count on \textit{Enron}, \textit{Contact}, and \textit{NDC}. On \textit{NDC} this is 1177160 variational scalars rather than 476175 point parameters. The factor memory remains linear, but quantified uncertainty has a visible constant cost. Table~\ref{tab:cap} shows that increasing $K$ from 2 to 16 lowers size error by $59.0\%$ and spectrum error by $18.4\%$. HYVINT mean RMSE instead reaches $0.0025$ at $K=4$ and rises by $100.0\%$ to $0.0050$ at $K=16$. DDE mean RMSE rises by $97.7\%$ from $K=2$ to $K=16$. More capacity therefore improves selected structural criteria without improving overall fidelity. Table~\ref{tab:int} shows intensity ratios of $1.82$ on synthetic data and $2.58$ on real \textit{Enron} data. The activity and degree correlations are $0.9857$ and $0.9758$, while the hyperedge activity and size correlations are $-0.0429$ and $0.0020$. Intensity separates observed incidences, but individual factors do not have a unique empirical interpretation. Supplementary Material~J and Supplementary Material~K give the complete diagnostics. All remaining intervals, tests, ablations, diagnostics, and efficiency results are in the Supplementary Material.

\section{Conclusion}
\label{sec:con}
In this paper, we propose HYVINT, an intensity-driven framework for hypergraph generation. The central idea is to model latent interaction intensities that induce binary node-hyperedge memberships through a Poisson link. This gives an explicit intensity-driven incidence formation mechanism, while the proposed lower-bound variational estimator makes the mechanism practical for learning node activity, hyperedge activity, and latent compatibility. On the theoretical side, our analysis separates the generation error into node-side estimation, latent factor distribution, and diffusion approximation components, and establishes a new bound for the first component. Through extensive evaluations, we demonstrate that HYVINT achieves strong fidelity while maintaining substantial novelty and diversity on synthetic and real-world hypergraphs.
\newpage
\bibliography{references}

\end{document}

%% file: HYVINT_AAAI_Original_Theory.tex
\section{Theoretical analysis}
\label{sec:the}
Let $\boldsymbol{\Theta}_n$ and $\mathcal A_n$ collect the node embeddings and activity parameters. For each hyperedge, let $R_{e_j}=(\boldsymbol\beta_{e_j},\rho_{e_j})$ and $\mathcal R_m=(R_{e_1},\ldots,R_{e_m})$, where $R_{e_j}\stackrel{\mathrm{i.i.d.}}{\sim}\mathbb P_R$. We consider $\widetilde m=m$ and study the estimation component
\begin{equation}
\begin{aligned}
\Delta_{(\boldsymbol\Theta_n,\mathcal A_n)}
&:=\mathbb E_{\mathbb P_{\mathcal R_m}}
\!\left[d_{\mathrm{KL}}\!\left(
\mathbb P_{\mathcal E\mid\mathcal R_m;\boldsymbol\Theta_n,\mathcal A_n}
\right.\right.\\[-2pt]
&\hspace{28mm}\left.\left.\|\,
\mathbb P_{\mathcal E\mid\mathcal R_m;\widehat{\boldsymbol\Theta}_n,\widehat{\mathcal A}_n}
\right)\right].
\end{aligned}
\label{eq:mai}
\end{equation}
The oracle identity and the complete KL bookkeeping decomposition are proved under Oracle KL Identity and KL Bookkeeping Identity in Supplementary Material~M.4. Let $\Lambda^\ast$ and $\widehat\Lambda$ denote the true intensity matrix and its clipped posterior mean estimator. Define
\begin{equation}
\begin{aligned}
L_{m,n}&=\log(m+n)+\log(1/\bar\lambda_{m,n}),\\
\epsilon_{m,n}^2&=K(m\vee n)\bar\lambda_{m,n}L_{m,n}.
\end{aligned}
\label{eq:mai2}
\end{equation}
The intensity class, estimator, assumptions, and proofs are given under First-Order Estimation Error in Supplementary Material~M.4.

\setcounter{theorem}{2}
\begin{theorem}
\label{thm:bri}
Under Assumption~M.2, let $\widehat\Lambda$ satisfy
\begin{equation}
c_1\bar\lambda_{m,n}
\le
\widehat\lambda_{ij}
\le
C_1\bar\lambda_{m,n},
\qquad \forall(i,j),
\end{equation}
where $0<c_1\le C_1<\infty$ and $C_1\bar\lambda_{m,n}\le1/2$. Then
\begin{equation}
\frac{\Delta_{(\boldsymbol\Theta_n,\mathcal A_n)}}{nm}
\lesssim
\frac{\|\widehat\Lambda-\Lambda^\ast\|_F^2}
{nm\bar\lambda_{m,n}}.
\end{equation}
\end{theorem}

Theorem~\ref{thm:bri} reduces the conditional generation error to intensity estimation. In particular, $\|\widehat\Lambda-\Lambda^\ast\|_F^2=o_p(nm\bar\lambda_{m,n})$ implies $\Delta_{(\boldsymbol\Theta_n,\mathcal A_n)}/nm=o_p(1)$. The remaining analysis therefore focuses on the contraction rate of $\widehat\Lambda$.

\begin{theorem}
\label{thm:fir}
Suppose Assumption~M.2 and Condition~M.7 hold, $\gamma_{m,n}^2=o(\epsilon_{m,n}^2)$, and the posterior tail satisfies Eq.~(M.113). For fixed $K$ as $m,n\to\infty$,
\begin{equation}
\frac{\Delta_{(\boldsymbol\Theta_n,\mathcal A_n)}}
{nm\bar\lambda_{m,n}}
=
O_p\!\left(
\frac{K L_{m,n}}{m\wedge n}
\right).
\end{equation}
\end{theorem}

Thus the normalized estimation component vanishes when $K L_{m,n}=o(m\wedge n)$. For fixed $K$, the smaller of $m$ and $n$ governs the polynomial part of the rate, while $\bar\lambda_{m,n}$ enters through $L_{m,n}$. When $m\asymp n$, the bound becomes $O_p(KL_{m,n}/n)$. Complete statements and proofs appear as Theorems~M.6 and~M.14 under First-Order Estimation Error in Supplementary Material~M.4.